\documentclass{article}

 \usepackage[final,nonatbib]{nips_2017}

\usepackage{amsmath, amssymb, graphicx, url, algorithm2e}
\usepackage{paralist}
\usepackage{booktabs}
\usepackage{multirow}
\usepackage[utf8]{inputenc} % allow utf-8 input
\usepackage[T1]{fontenc}    % use 8-bit T1 fonts
\usepackage{hyperref}       % hyperlinks
\usepackage{url}            % simple URL typesetting
\usepackage{booktabs}       % professional-quality tables
\usepackage{amsfonts}       % blackboard math symbols
\usepackage{nicefrac}       % compact symbols for 1/2, etc.
\usepackage{microtype}      % microtypography
\usepackage{amsmath, amsthm, amssymb, multirow, paralist}

\newtheorem{thm}{Theorem}
\newtheorem{prop}{Proposition}
\newtheorem{lem}{Lemma}

 \newtheorem{ass}{Assumption}

\title{A Simple Analysis for Exp-concave  Empirical  Minimization  with Arbitrary Convex Regularizer}

\author{
 Tianbao Yang$^\dagger$, Zhe Li$^\dagger$, Lijun Zhang$^\ddagger$\\
  $^\dagger$Department of Computer Science, The University of Iowa\\
  $^\ddagger$National
Key Laboratory for Novel Software Technology, Nanjing University\\
    \texttt{tianbao-yang, zhe-li-1@uiowa.edu, zhanglj@lamda.nju.edu.cn} \\
}

\def \E {\mathrm{E}}
\def \x {\mathbf{x}}

\def \z {\mathbf{z}}
\def \u {\mathbf{u}}
\def \H {\mathcal{H}}
\def \w {\mathbf{w}}
\def \R {\mathbb{R}}

\def \Z {\mathcal{Z}}

\def \W {\mathcal{W}}
\def \N {\mathcal{N}}

\def \v {\mathbf{v}}

\def \B {\mathcal{B}}

\def \Fh {\widehat{F}}

\def \wt {\widetilde{\w}}

\def \P {\mathbb{P}}

\def \wh {\widehat{\w}}

\def \B {\mathcal{B}}

\begin{document}

\maketitle

\begin{abstract}
In this paper, we present a simple analysis of {\bf fast rates} with {\it high probability} of {\bf empirical  minimization} for {\it stochastic composite optimization} over a finite-dimensional bounded convex set with exponential concave loss functions and  an arbitrary convex regularization. To the best of our knowledge, this result is  the first of its kind. As a byproduct, we can directly obtain the fast rate  with {\it high probability} for exponential concave empirical risk minimization with and without any convex regularization, which not only extends existing results of empirical risk minimization but also provides a unified framework for analyzing exponential concave empirical  risk minimization with and without {\it any} convex regularization.  Our proof is very simple only exploiting  the covering number of a finite-dimensional bounded set and a concentration inequality of random vectors. 
\end{abstract}

\section{Introduction}
Stochastic minimization with exponential concave (or exp-concave in short)  loss functions can find many applications in machine learning, e.g., linear regression, logistic regression, support vector machine with squared hinge loss and portfolio optimization.  There are two popular approaches for stochastic optimization. The first approach is called {\it Sample Average Approximation} (also known as empirical risk minimization in machine learning~\cite{vapnik-1998-statistical}), in which a set of $n$ i.i.d examples are drawn from the underlying distribution and an empirical risk minimization problem is solved. The second approach is {\it Stochastic Approximation}~\cite{SA:Springer} (closely related to online optimization), which iteratively learns the model from  randomly sampled examples.  Comparing with stochastic approximation, empirical risk minimization is deemed as more general and usually achieves better performance in practice. Importantly, it is amenable to any optimization algorithms. 

Fast rates of optimization with exponential concave functions in online setting or in stochastic setting have attracted a bulk of studies. In the seminal paper by Hazan et al.~\cite{ML:Hazan:2007}, the authors proposed an online Newton step (ONS) algorithm - a reminiscent of Newton-Raphson method for offline optimization, which achieves an $O(d\log n)$ regret bound with $d$ being the dimension of the problem and $n$ being the total number of iterations. With the standard trick of online-to-batch conversion~\cite{TIT04:Bianchi,DBLP:conf/nips/KakadeT08}, one can obtain a fast convergence rate of $O(d\log n/n)$ recently achieved in stochastic setting~\cite{arXiv:1605.01288,DBLP:conf/colt/Mahdavi0J15}.  However, the computational cost of ONS scales badly with the dimensionality of the problem (with a $d^4$ factor)~\cite{DBLP:conf/nips/KorenL15}, which may prohibit its application to high-dimensional problems. 

In terms of  empirical risk minimization (ERM), it was not until recently that the fast rates for exp-concave risk minimization were established. Koren \& Levy~\cite{DBLP:conf/nips/KorenL15} obtained  the first result  for exp-concave risk minimization  that ERM is able to attain fast generalization rates, i.e., an $O(d/n)$ expected convergence bound  - difference between the risk of the learned model and the risk of the optimal model. Strictly speaking, their guarantee  is not for the solution to ERM but for the solution to a penalized ERM by adding a strongly convex regularizer to the ERM objective. % In addition, their fast rate result only holds in expectation. 
Gonen and Shalev-Shwartz~\cite{DBLP:journals/corr/abs/1601.04011} derived a similar expectational fast rate for supervised learning with exp-concave losses.   Recently, Mehta~\cite{arXiv:1605.01288} established high probability fast rates for exp-concave empirical risk minimization, which is only worse by a factor of $\log(n)$ than the in-expectation rate. %He also managed to remove the $\log(n)$ factor by using a boosting technique based on the expectational results in~\cite{DBLP:conf/nips/KorenL15,DBLP:journals/corr/abs/1601.04011}. 

This paper is motivated by solving the following {\bf stochastic composite optimization} problem:
\begin{align}\label{eqn:opt}
\w_* = \arg\min_{\w\in\W} \left[P(\w) \triangleq \E_{\z\sim \mathcal P}[f(\w, \z)] + R(\w)\right],
\end{align}
where the objective consists of a stochastic component that is the expectation over a random function $f(\w, \z)$ and a deterministic component $R(\w)$.  In this paper, we will assume: (i) $\W\subseteq\R^d$ is a compact and convex set; (ii) $f(\w,\z)$ is a smooth and $\beta$-exp-concave function of $\w$ for any $\z$, and is Lipschitz continuous over the bounded domain $\W$.  To make it general, we do not impose strong convexity or exp-concave or smoothness assumption on $R(\w)$ except for convexity. 

We study the convergence of  the {\bf empirical minimizer} of~(\ref{eqn:opt}), i.e., 
\begin{align}\label{eqn:opt2}
\wh = \arg\min_{\w\in\W} \left[P_n(\w) \triangleq \frac{1}{n}\sum_{i=1}^nf(\w, \z_i) + R(\w)\right],
\end{align}
 where $\z_1, \ldots, \z_n$ are i.i.d samples from $\P$. Our major goal here is to establish the fast convergence rate of the empirical minimizer in terms of $P(\wh) - P(\w_*)$.  This is in contrast to many previous works focusing on convergence analysis of stochastic approximation algorithms~\cite{Lan:2010:Optimal}  for solving~(\ref{eqn:opt}).   It is noticeable that  many efficient optimization algorithms are available for solving~(\ref{eqn:opt2})~\cite{xiao2014proximal,DBLP:conf/nips/DefazioBL14}.  In machine learning applications, the deterministic component $R(\w)$ is usually a regularizer that enforces  some kind of structure over the model $\w$. Many studies in machine learning and statistics have found that using a certain kind of regularization that incorporates  prior knowledge about the model can lead to great improvements on performance in many applications. %Commonly used structured regularizers include $\ell_1$ norm, group lasso, sparse group lasso, graphic guided lasso, fused lasso, exclusive lasso, $\ell_{1,2}$ norm, $\ell_{1, \infty}$ norm, the trace norm of a matrix, and etc~\cite{Bach_2012b}. All these regularizers are usually used  under  some  prior  knowledge of the problem to restrict the search space of the optimal model. 

To establish the convergence rate of the empirical minimizer~(\ref{eqn:opt2}), one may consider to define a new loss $h(\w, \z) = f(\w, \z) + R(\w)$ such that $P(\w) = \E_{\z}[h(\w, \z)]$ and $P_n(\w) = \frac{1}{n}\sum_i h(\w, \z_i)$, and then leverage the existing theory to prove the convergence rate.  However, the combined function $h(\w, \z)$ is not necessarily   an exp-concave function (see {\bf Example 2} below). Therefore, all the previous fast rates analysis for exp-concave empirical risk minimization cannot carry over to the considered problem. As a consequence,  the standard generalization theory of ERM~\cite{COLT:Shalev:2009}  applied to~(\ref{eqn:opt2}) can only guarantee an $O(\sqrt{d/n})$ convergence rate, which is worse than $O(d\log n/n)$ - the fast rate that we aim to establish. 

\textbf{Contributions.} The main contribution of this work is  a simple analysis of a fast rate of $O(d\log n/n + d\log(1/\delta)/(n\beta))$ with high probability $1-\delta$ for the  empirical  minimization~(\ref{eqn:opt2}) with a $\beta$-exp-concave losses $f(\w, \z)$ and an arbitrary convex regularizer $R(\w)$ in terms of  $P(\wh) - P(\w_*)$. Our proof is simple and elementary, which only utilizes  the covering number of $\W$ and a concentration inequality of random vectors. 

\section{Comparison with Previous Works}
There are extensive studies about fast rates of ERM. Due to limit of space, the review below focuses on closely related work.   %Here we focus on the comparison of our result with  results on ERM for exp-concave  risk minimization~\cite{DBLP:conf/nips/KorenL15,DBLP:journals/corr/abs/1601.04011,arXiv:1605.01288} and (regularized) strongly convex  risk minimization~\cite{COLT:Shalev:2009,DBLP:journals/corr/0005YJ17}. 
In the three recent studies~\cite{DBLP:conf/nips/KorenL15,DBLP:journals/corr/abs/1601.04011,arXiv:1605.01288}, the focus is to establish fast rates in terms of risk minimization without a regularizer, i.e, 
\begin{align}\label{eqn:opt3}
\min_{\w\in\W} F(\w) \triangleq \E_{\z}[f(\w, \z)],
\end{align}
where $f(\w, \z)$ is a $\beta$-exp-concave function. As reasoned above, the fast rates in these studies do not carry over to the minimization problem~(\ref{eqn:opt}) with an arbitrary regularizer. 

 Koren \& Levy~\cite{DBLP:conf/nips/KorenL15} studied the convergence of a penalized/regularized empirical risk minimizer by 
\begin{align}\label{eqn:opt4}
\wt=\arg\min_{\w\in\W} \left[\Fh_n(\w) \triangleq \frac{1}{n}\sum_{i=1}^nf(\w, \z_i) + \frac{1}{n}g(\w)\right]. 
\end{align}
They assumed that $g(\w)$ is a $1$-strongly regularizer  w.r.t the Euclidean norm and is bounded over $\W$.  Gonen and Shalev-Shwartz~\cite{DBLP:journals/corr/abs/1601.04011}  focused on the risk minimization with generalized linear model : 
\begin{align}
\min_{\w\in\W} L(\w) \triangleq \E_{(\x, y)\sim D}[\phi_y(\w^{\top}\x)], 
\end{align}
which is a special case of the general minimization problem~(\ref{eqn:opt}). Under the assumption that $\phi_y(\w^{\top}\x)$ is a $\beta$-exp-concave function of $\w$, they provided an expected convergence rate of the empirical risk minimizer, which is in the same order as the result in~\cite{DBLP:conf/nips/KorenL15}, i.e., $O(d/(n\beta))$. 

There are three differences between our work and these two works: (i)  their results of fast rate are with respect to $\min_{\w\in\W}F(\w)$, where $F(\w)$ does not include any regularizer; in contrast our result is respect to $\min_{\w\in\W}F(\w) + R(\w)$; (ii) the strongly convex penalization term $g(\w)$ in~\cite{DBLP:conf/nips/KorenL15} is artificially added to the ERM to facilitate the analysis; in contrast the arbitrary convex regularizer $R(\w)$ in this paper is built into the objective; (iii) their fast rate guarantee is in expectation while our fast rate guarantee is in high probability. In light of these differences, we can see that our result is more general and much stronger. In particular, when setting $R(\w)=0$ in our problem, we obtain the fast rate with high probability of the empirical risk minimizer, which is only worse by a factor of $\log(n)$ than the in-expectation rate in~\cite{DBLP:journals/corr/abs/1601.04011,DBLP:conf/nips/KorenL15}. Additionally, a similar high probability risk bound {\bf with respect to $\min_{\w\in\W}F(\w)$ for any regularized empirical risk minimizer~(\ref{eqn:opt4})} can be easily derived in our framework as long as $g(\x)$ is convex and  bounded over $\W$ (see Theorem~\ref{thm:smooth:convex2}). 

A more recent work by Mehta~\cite{arXiv:1605.01288} establishes a fast rate of $O(d\log n/(n\beta ))$ for the exp-concave ERM.  His analysis for the empirical risk minimizer is based on the connection between exp-concavity and the stochastic mixability  condition, and he exploited the heavy machinery developed in their previous work for fast rate analysis of empirical risk minimizer under the stochastic mixability condition~\cite{DBLP:conf/nips/MehtaW14}. However, this analysis does not apply to the regularized empirical risk minimizer~(\ref{eqn:opt4}) with an arbitrary convex regularizer. Admittedly, \cite{arXiv:1605.01288} has made additional contribution on removing the $\log (n)$ factor by using boosting techniques to boost the in-expectation results of~\cite{DBLP:conf/nips/KorenL15,DBLP:journals/corr/abs/1601.04011}.  %~\footnote{ A minor thing might be worth of attention is that, a careful investigation of the proof of his Theorem 1 for empirical risk minimization requires  the number of samples $n$ to be larger than some unknown constant (related to $\eta_f$ in his proof).} % The strategy is to divide the data sets into two groups with half of the $n$ samples, and further divide one group into a small number of $k$ sub-groups ($k=\log(1/\delta)$),  and then run the regularized ERM of Koren \& Levy for general learning or ERM of Gonen \& Shalev-Shwartz for supervised learning on the $k$ sub-groups to learn $k$ models, and finally select one model that performs the best on another half of examples. 

We comment on the extra conditions on the loss functions $f(\w,\z)$. \cite{DBLP:conf/nips/KorenL15} requires that $f(\w,\z)$ is a smooth function of $\w$ for any $\z$ and is bounded over $\W$. %, i.e., $|f(\w, \z) - f(\w', \z)|\leq C$ for all $\w, \w'\in\W$ for some $C>0$. 
Both \cite{DBLP:journals/corr/abs/1601.04011} and \cite{arXiv:1605.01288}  require that $\phi_y(z)$ or $f(\w,\z)$ to be Lipschitz continuous. Note that Lipschitz continuity  implies bounded range of the loss function $f(\w,\z)$.   In the present paper, we assume that $f(\w,\z)$ is Lipschitz continuous and smooth over $\W$ for all $\z$. Both conditions are necessary for us to deliver a simple analysis for exp-concave empirical minimization with an arbitrary convex regularizer. We also notice for a twice differentiable smooth and exp-concave function, Lipschitz continuity is automatically  satisfied. 

%Under Lipschitz continuity  condition of the loss function $f(\w, \z)$, strong convexity is a stronger condition than exp-concavity as the former implies the latter. 
Next, we briefly mention several  results regarding fast rates of ERM under strong convexity condition - a stronger condition than exp-concavity. Shalev-Shwartz et al.~\cite{COLT:Shalev:2009}  established an $O(1/(n\beta))$ in-expectation convergence bound of ERM  over any bounded convex set  for~(\ref{eqn:opt3}), which requires each individual loss function $f(\w, \z)$ to be a $\beta$-strongly convex function of $\w$. However, in machine learning applications, individual loss functions are usually not strongly convex.
Recently, Zhang et al.~\cite{DBLP:journals/corr/0005YJ17} developed optimistic rates of ERM over a bounded convex set $\W\subseteq\R^d$  for~(\ref{eqn:opt3}), where they assumed $f(\w, \z)$ to be non-negative and smooth. Under $\beta$-strong convexity assumption of $F(\w)$, they established a fast rate of $O(d\log (n)/(n\beta))$ with high probability and a faster rate of $O(1/(n^2\beta))$ when the optimal risk is small and the number of samples is sufficiently large ($n\geq \Omega(d\log n/\beta)$).  In~\cite{NIPS2008_3400}, the authors considered the composite problem~(\ref{eqn:opt}) with $f(\w, \z)$ having a generalized linear form and established a fast rate with high probability of $O(1/(n\beta))$ for $\beta$-strongly convex objective.  
%When applying the analysis in the above works to the regularized problem~(\ref{eqn:opt}) without the strong convexity assumption of the loss function, one may assume that the regularizer $R(\w)$ or the expected function $\E_{\z}[f(\w, \z)]$  or their sum is strongly convex, which is more restricted than our result. 

We can also compare with stochastic approximation algorithms. Lan~\cite{Lan:2010:Optimal} presented an optimal method for solving~(\ref{eqn:opt}) without the exp-concavity assumption,  which employs a proximal mapping to handle $R(\w)$ and has a convergence rate of $O(\frac{1}{n} + \frac{\tau}{\sqrt{n}})$, where  $\tau$ is related to the noise in the stochastic gradient.  In contrast, the convergence rate of empirical minimizer shown in this work has a better dependence on $n$. One may also apply the online-to-batch conversion to a variant of ONS that employs the proximal mapping to handle $R(\w)$ to obtain an $O(d\log (n)/(n\beta))$ convergence with high probability. Nonetheless, the resulting algorithm will be at least as expensive as ONS. 
Finally, it is worth mentioning that the linear dependence on the dimensionality $d$ of the convergence rate of the empirical minimizer  for the minimization problem~(\ref{eqn:opt}) over $\W\subseteq\R^d$ is unavoidable even with  smooth  functions $f(\w, \z)$~\cite{NIPS2016_ERM}.

\section{Preliminaries}
In this section, we present some preliminaries. Let $\z\in\Z$ denote a random variable following a distribution $\P$. Denote by $\nabla f(\w, \z)$ the partial gradient in terms of $\w$. Define 
\begin{align}
F(\w) = \E_{\z}[f(\w, \z)], \quad P(\w) = F(\w) + R(\w).
\end{align}
Let $\|\cdot\|_2$ denote the Euclidean norm of a vector. For a positive definite matrix $H\succ 0$, define the $H$-norm $\|\w\|_H = \sqrt{\w^{\top}H\w}$ and its dual norm $\|\w\|_{H^{-1}}=\sqrt{\w^{\top}H^{-1}\w}$. By H\"{o}lder's inequality, we have  $\w^{\top}\u\leq \|\w\|_H\|\u\|_{H^{-1}}$.

A function $f(\w)$ is $\beta$-exp-concave over the domain $\W$ for some $\beta>0$ if the function $\exp(-\beta f(\w))$ is concave over $\W$. If $f(\w)$ is twice differentiable and $\beta$-exp-concave, it follows that
 \begin{align}\label{eqn:psd}
 \nabla^2 f(\w)\succeq \beta \nabla f(\w)\nabla f(\w)^{\top}.
 \end{align} A function $f(\w)$ is a $L$-smooth function with respect to $\|\cdot\|_2$ over $\W$ if the following inequality  holds that for all $\w, \u\in\W$ for some $L>0$
\begin{align}
f(\w)\leq f(\u) + \nabla f(\u)^{\top}(\w - \u) + \frac{L}{2}\|\w - \u\|_2^2.
\end{align}
A function $f(\w)$ is $G$-Lipschitz continuous if  $|f(\w) - f(\u)|\leq G\|\w - \u\|_2, \forall \w, \u\in\W$. %A function $g(\w)$ is a strongly convex function with respect to $\|\cdot\|_2$ over $\W$ if the following inequality  holds that for all $\w, \u\in\W$ for some $\lambda>0$
%\begin{align}
%g(\w)\geq g(\u) + \partial g(\u)^{\top}(\w - \u) + \frac{\lambda}{2}\|\w - \u\|_2^2, 
%\end{align}
%where $\lambda$ is called the strong convexity modulus of $g(\w)$. 

We will make the following assumptions regarding the loss function $f(\w, \z)$ and the regularizer $R(\w)$. 
\begin{ass}\label{ass:1}
We assume that  (i) $\W$ is a closed and bounded convex set, i.e., there exists $R$ such that $\|\w\|_2\leq R$ for all $\w\in\W$. (ii) $f(\w, \z)$ is a $G$-Lipschitz continuous, $L$-smooth and $\beta$-exp-concave  function of $\w\in\W$ for any $\z\in\Z$. 
%\item $f(\w, \z)$ is a $G$-Lipschitz continuos function of $\w$ for any $\z\in\Z$.
%\item $f(\w,\z)$ is an $\beta$-exp-concave function of $\w$ for any $\z\in\Z$.
(iii) $R(\w)$ is a convex function. 
\end{ass}
{\bf Remark 1:} Note that if $f(\w,\z)$ is twice differentiable, the smoothness and exp-concavity naturally imply  Lipschitz continuity. This can be seen from~(\ref{eqn:psd}) by noting that $\nabla^2 f(\w)\preceq LI$. As a result $\|\nabla f(\w)\|_2\leq \sqrt{L/\beta}$, which implies that $f(\w, \z)$ is  $G=\sqrt{L/\beta}$-Lipschitz continuous.

There are many machine learning problems  satisfying the above assumptions. If we consider the loss function in supervised learning $f(\w, \z) = \phi(\w^{\top}\x, y)$ where $(\x, y)\sim D$ denote a random feature vector and label, then the square loss, logistic loss, squared hinge loss are exp-concave function under appropriate conditions on $(\x, y)$. Let us consider the square loss  $f(\w,(\x, y)) = (\w^{\top}\x - y)^2$ as an example. 

{\bf Example 1}. Suppose $\x$ and $y$ are bounded. W.l.g we can assume $\|\x\|_2\leq 1$ and $|y|\leq R$. Then $\nabla f(\w, (\x, y)) = 2(\w^{\top}\x - y)\x$ and $\nabla^2 f(\w, (\x, y)) = 2\x\x^{\top}$. It then follows that for any $\beta\leq\frac{1}{8R^2} $ and any $\w\in\W$, 
\[
\beta \nabla f(\w, (\x, y))\nabla f(\w, (\x, y))^{\top}\preceq\beta 16R^2\x\x^{\top}\preceq \nabla^2 f(\w, (\x, y)),
\]
which guarantees that $\exp(-\beta f(\w, (\x, y)))$ is a $\beta$-exp-concave function of $\w$. 

Next, we give an example showing that the sum of an exp-concave function and a convex function is not necessarily an exp-concave function. 

{\bf Example 2}. Let  $\w=(w_1, w_2)^{\top}\in[0,1]^2$,  $f(\w)  = (w_1-1)^2$ and $R(\w) = w_2$. To see $f(\w)$ is an exp-concave function, we can show that $\nabla f(\w) = (2(w_1-1), 0)^{\top}$ and $\nabla^2 f(\w) = \left(\begin{array}{cc}2& 0 \\ 0 & 0\end{array}\right)$, then for any $\beta\leq 1/2$ and $\w\in[0, 1]^2$
\[
\beta\nabla f(\w)\nabla f(\w)^{\top} =\beta\left(\begin{array}{cc}4(w_1 - 1)^2& 0 \\ 0 & 0\end{array}\right) \preceq \nabla^2 f(\w).
\]
To see $P(\w)  =f(\w) + R(\w)$ is not an exp-concave function, we can show that $\nabla P(\w) = (2(w_1-1), 1)^{\top}$ and  $\nabla^2 P(\w) = \left(\begin{array}{cc}2& 0 \\ 0 & 0\end{array}\right)$, then  for any $\beta>0$ the following matrix is not  positive semi-definite 
\[
 \nabla^2 P(\w) - \beta\nabla P(\w)\nabla P(\w)^{\top}   = \left(\begin{array}{cc}2& 0 \\ 0 & 0\end{array}\right)-\beta\left(\begin{array}{cc}4(w_1 - 1)^2& 0 \\ 0 & 1\end{array}\right) =  \left(\begin{array}{cc}2 - 4(w_1-1)^2& 0 \\ 0 & -\beta\end{array}\right),
\]
which contracts to~(\ref{eqn:psd}) if $P(\w)$ is an exp-concave function. As a result, $P(\w)$ is not an exp-concave function.  

In our analysis, we will use the covering number of $\W$. A subset $\N(\W,\varepsilon) \subseteq \W$ is called an $\varepsilon$-net of $\W$ if for every $\w \in \W$ one can find $\w' \in \N(\W, \varepsilon)$ so that $\|\w - \w'\|_2 \leq \varepsilon$. The minimal cardinality of an $\varepsilon$-net of $\W$ is called the covering number and denoted by $N(\W,\varepsilon)$. The covering number of the Euclidean ball $\B^d_2=\{\w\in\R^d: \|\w\|_2\leq 1\}$ can be estimated using a standard volume comparison argument~\cite{Convex:body:89}, as follows
%\[
$N(\B^d_2, \varepsilon)\leq \left(\frac{3}{\varepsilon}\right)^d, \: \varepsilon\in(0,1).$
%\]
The covering numbers are (almost) increasing by inclusion~\cite{OneBit:Plan:LP}: $\W \subseteq \B$ implies $N(\W, 2\varepsilon) \leq  N(\B, \varepsilon)$.  Since $\W\subseteq\B_2^d(R)=\{\w\in\R^d: \|\w\|_2\leq R\}$, then 
\begin{align}
N(\W, \varepsilon)\leq N(\B^d_2(R), \varepsilon/2) \leq\left(\frac{6R}{\varepsilon}\right)^d.
\end{align}

Finally, we present two basic lemmas, which  will be useful  in our analysis. 
\begin{lem}\cite{ML:Hazan:2007}\label{lem:0}
Under {\bf Assumption~\ref{ass:1} }, for any $\w, \w'\in\W$, and  $\z\in\Z$ the following holds for $\sigma \leq \frac{1}{2}\min\left\{\frac{1}{8GR}, \beta\right\}$ 
\begin{align*}
f(\w, \z) \geq f(\w', \z) + (\w - \w')^{\top}\nabla f(\w', \z) + \frac{\sigma}{2}(\w - \w')^{\top}\nabla f(\w', \z)\nabla f(\w', \z)^{\top}(\w - \w').
\end{align*}
\end{lem}
\begin{lem}\label{lem:opt}Let $\w_*$ be an optimal solution to~(\ref{eqn:opt}). Then for any $\w\in\W$ we have
\begin{align*}
(\w - \w_*)^{\top}\nabla F(\w_*)\geq R(\w_*) - R(\w).
\end{align*}
\end{lem}
The proof of Lemma~\ref{lem:0} can be found in~\cite{ML:Hazan:2007} and is thus omitted. The proof of Lemma~\ref{lem:opt} is simple and presented below.
\begin{proof}[Proof of Lemma~\ref{lem:opt}]
By the optimality condition of~(\ref{eqn:opt}), there exists $\v_*\in\partial R(\w_*)$ such that  $(\w - \w_*)^{\top}(\nabla F(\w_*)+ \v_*)\geq 0$. Since $R(\w)$ is a convex function, then we have $R(\w)\geq R(\w_*) + (\w - \w_*)^{\top}\v_*$. 
Combining the above two inequalities, we have $(\w - \w_*)^{\top}\nabla F(\w_*)\geq R(\w_*) - R(\w)$. 
\end{proof}

\section{Main Result and Analysis}
In the sequel, we let $\sigma$ be a constant as in Lemma~\ref{lem:0} and fix $\w_*$ - an optimal solution of~(\ref{eqn:opt}). Our main result is stated in the following theorem. 
\begin{thm} \label{thm:smooth:convex}
For the stochastic composite minimization problem~(\ref{eqn:opt}), we consider the empirical minimizer $\wh$ by solving~(\ref{eqn:opt2}). Under {\bf Assumptions~\ref{ass:1}}, with probability at least $1 - \delta$, we have
\begin{equation*}
\begin{split}
 P(\wh) - P(\w_*) \leq & O\left(\frac{d\log n}{n}+ \frac{d\log(1/\delta)}{n\sigma}\right).
\end{split}
\end{equation*}
%where $\sigma$ is a constant as in Lemma~\ref{lem:0}.
\end{thm}
{\bf Remark 2:} Note that when $R(\w)=0$, we directly obtain a fast rate with high probability of the empirical risk minimizer for the exp-concave risk minimization problem~(\ref{eqn:opt3}). We can also obtain a similar result  for~(\ref{eqn:opt3}) regarding the regularized empirical risk minimizer~(\ref{eqn:opt4}) that provides a different way for solving~(\ref{eqn:opt3}), which is usually preferred over solving the empirical risk minimization problem without any regularization due to that (i) a regularization can lead to a better condition from the perspective of optimization complexity;  and (ii) the prior knowledge about the model can be encoded into the regularizer. 
\begin{thm} \label{thm:smooth:convex2}
For the risk minimization problem~(\ref{eqn:opt3}), we consider the regularized empirical risk minimizer $\wt$ by solving~(\ref{eqn:opt4}). Under {\bf Assumptions~\ref{ass:1}}(i), (ii), and that $g(\x)$ is bounded over $\W$ such that $\sup_{\w, \w'\in\W}|g(\w) - g(\w')|\leq B$, with probability at least $1 - \delta$, we have
\begin{equation*}
\begin{split}
 F(\wt) - \min_{\w\in\W}F(\w)\leq & O\left(\frac{d\log n}{n}+ \frac{d\log(1/\delta)}{n\sigma}\right).
\end{split}
\end{equation*}
%where $\sigma$ is a constant as in Lemma~\ref{lem:0}.
\end{thm}
{\bf Remark 3:} This new result not only addresses the open problem raised in~\cite{DBLP:conf/nips/KorenL15} about the high probability bound for the strongly regularized empirical risk minimizer  but also extends the fast rate to any regularized empirical risk minimizer  as long as the regularizer is convex. In comparison, \cite{DBLP:conf/nips/KorenL15}  only provides the expectational fast rate for the  regularized empirical risk minimizer using a strongly convex regularizer.  The additional assumption used in our analysis compared to \cite{DBLP:conf/nips/KorenL15} is the Lipschitz continuity of the loss functions over the domain $\W$, which is mild. 

%{\bf Remark 3:} The $d$ factor in Theorem~\ref{thm:smooth:convex} and~\ref{thm:smooth:convex2} comes from bounding the covering number of $\W\in\R^d$, which is a worst-case analysis. We can understand from this viewpoint how regularizers that encode prior knowledge about the optimal model $\w_*$ can yield faster convergence. In particular, using a regularizer in the empirical minimization can restrict the solution $\wh$ to a smaller set  that contains $\w_*$, which could have a much smaller covering number. Then our analysis can utilize 

Below, we will prove Theorem~\ref{thm:smooth:convex} and Theorem~\ref{thm:smooth:convex2}. To prove the theorems, we first establish several lemmas.
\begin{lem}
Suppose  {\bf Assumptions~\ref{ass:1}} hold. For any $\alpha>0$ we have
\begin{align}
\frac{\alpha}{2}\|\w - \w_*\|_H^2\leq P(\w) - P(\w_*) + \frac{\alpha}{2}\|\w - \w_*\|_2^2,
\end{align}
where $H = I  + \frac{\sigma}{\alpha}\E[\nabla f(\w_*, \z) \nabla f(\w_*, \z)^{\top}]$. 
\end{lem}
\begin{proof}
Let $F(\w) = \E_{\z}[f(\w,\z)]$.  We begin with the following inequality in Lemma~\ref{lem:0} 
\begin{align*}
f(\w, \z) \geq f(\w', \z) + (\w - \w')^{\top}\nabla f(\w', \z) + \frac{\sigma}{2}(\w - \w')^{\top}\nabla f(\w', \z)\nabla f(\w', \z)^{\top}(\w - \w')
\end{align*}
Let $\w' = \w_*$ and taking expectation over both sides over the random variable $\z\sim \P$, we have
\begin{align*}
F(\w)\geq F(\w_*) + (\w - \w_*)^{\top}\nabla F(\w_*) + \frac{\sigma}{2}(\w - \w_*)^{\top}\E[\nabla f(\w_*, \z) \nabla f(\w_*, \z)^{\top}](\w - \w_*)
\end{align*}
%By the optimality condition of~(\ref{eqn:opt}), there exists $\v_*\in\partial R(\w_*)$ such that 
%\begin{align}\label{eqn:oc}
%(\w - \w_*)^{\top}(\nabla F(\w_*)+ \v_*)\geq 0.
%\end{align}
%Since $R(\w)$ is a convex function, then we have
%\begin{align}\label{eqn:R}
%R(\w)\geq R(\w_*) + (\w - \w_*)^{\top}\v_*
%\end{align}
Adding up the above inequality and the inequality in Lemma~\ref{lem:opt}, we have
\begin{align*}
P(\w) \geq P(\w_*) + \frac{\sigma}{2}(\w - \w_*)^{\top}\E[\nabla f(\w_*, \z) \nabla f(\w_*, \z)^{\top}](\w - \w_*)
\end{align*}
Adding $ \frac{\alpha}{2}\|\w - \w_*\|_2^2$ on both sides and by the definition of $H$, we can finish the proof. % As a result,
%\begin{align*}
%P(\w) - P(\w_*) + \frac{\alpha}{2}\|\w - \w_*\|_2^2\geq \frac{\alpha}{2}(\w - \w_*)^{\top}\left(I + \frac{\sigma}{\alpha}\E[\nabla f(\w_*, \z) \nabla f(\w^*, \z)^{\top}]\right)(\w - \w_*)
%\end{align*}
%By the definition of $H$ we can finish the proof. 
%\[
%\frac{\alpha}{2}\|\w - \w_*\|_H^2\leq P(\w) - P(\w_*) + \frac{\alpha}{2}\|\w - \w_*\|^2
%\]

\end{proof}

\begin{lem} \label{lem:vec:con} Let   $H = I  + \frac{\sigma}{\alpha}\E[\nabla f(\w^*, \z) \nabla f(\w^*, \z)^{\top}]$.  Under {\bf Assumption~\ref{ass:1}}, with probability at least $1-\delta$, for any $\alpha>0$, we have
\begin{equation} \label{eqn:add:2}
\left\|\nabla P(\w_*) - \nabla P_n(\w_*)\right\|_{H^{-1}} \leq \frac{2G\log(2/\delta)}{n} + \sqrt{\frac{2 \alpha d\log(2/\delta)}{n\sigma}}. 
\end{equation}
\end{lem}
\begin{proof} To prove the above lemma. We need the following concentration result of random vectors. 
\begin{prop} \cite{Smale:learning}. \label{lem:con} Let $\H$ be a Hilbert space equipped with a norm $\|\cdot\|$ and let $\xi$ be a random variable with values in $\H$. Assume $\|\xi\|\leq M < \infty$ almost surely. Denote $\sigma^2(\xi)=\E\left[\|\xi\|^2\right]$. Let  $\{\xi_i\}_{i=1}^m$ be $m$ ($m < \infty$) independent drawers of $\xi$. For any $0 < \delta < 1$, with confidence $1-\delta$,
\[
\left\| \frac{1}{m} \sum_{i=1}^m \left[\xi_i -\E[\xi_i]\right] \right\| \leq \frac{2 M \log(2/\delta)}{m} + \sqrt{\frac{2 \sigma^2(\xi) \log(2/\delta)}{m}}.
\]
\end{prop}
To utilize the above lemma, we consider $\nabla f(\w_*, \z)\in\R^d$ as a random variable in a Hilbert space equipped with a norm $\|\nabla f(\w_*, \z)\|=\|\nabla f(\w_*, \z)\|_{H^{-1}}$. Then we have $\E[\nabla f(\w_*, \z)] = \nabla F(\w_*)$.  To prove Lemma~\ref{lem:vec:con}, we need an upper bound of $\E\left[\|\nabla f (\w_*, \z)\|_{H^{-1}}^2\right]$ and $\|\nabla f(\w_*, \z)\|_{H^{-1}}$. First, we note that $H\succ I$, then $\|\nabla f_i (\w_*)\|_{H^{-1}}\leq \|\nabla f_i(\w_*)\|_2\leq G $. Second, 
\begin{align*}
&\E\left[\|\nabla f (\w_*,\z)\|_{H^{-1}}^2\right] = \text{tr}(H^{-1}\E[\nabla f(\w_*, \z)\nabla f(\w_*, \z)^{\top}])\leq \frac{\alpha}{\sigma}d,
\end{align*}
where $\text{tr}(\cdot)$ denotes the trace  function and the last inequality  uses $ \frac{\alpha}{\sigma}H\succ \E[\nabla f(\w_*, \z)\nabla f(\w_*, \z)^{\top}]$. 
Then, according to Proposition~\ref{lem:con}, with probability at least $1-\delta$, we have
\begin{align*}
 \left\|\nabla P(\w_*) - \nabla P_n(\w_*)\right\|_{H^{-1}}& = \left\|\nabla F(\w_*) - \frac{1}{n}\sum_{i=1}^n \nabla f(\w_*, \z_i) \right\|_{H^{-1}}\\
 &\leq  \frac{2G\log(2/\delta)}{n} + \sqrt{\frac{2\alpha d   \log(2/\delta)}{n\sigma}}.
\end{align*}

\end{proof}

\begin{lem} \label{lem:net} Under {\bf Assumptions~\ref{ass:1}}, with probability at least $1-\delta$, for any $\w \in \W$ and any $\varepsilon>0$, we have
\begin{align*}
&\left\| \nabla P(\w) - \nabla P(\w^*) -  [\nabla P_n(\w) - \nabla P_n(\w^*)]\right\|_2  \leq \frac{LC(\varepsilon) \|\w - \w^*\|_2 }{n}  \\
&+ \frac{LC(\varepsilon)\varepsilon }{n} + \sqrt{\frac{LC(\varepsilon) (P(\w)-  P(\w_*) )}{n} }     + \sqrt{\frac{ LGC(\varepsilon) \varepsilon }{n}} 
 + 2L\varepsilon.
\end{align*}
where $C(\varepsilon) = 4(\log(2/\delta) + d\log(6R/\varepsilon))$. 
\end{lem}
The proof of the above lemma  is similar to the proof of Lemma 1 in~\cite{DBLP:journals/corr/0005YJ17} and is deferred to supplement.  The idea of the proof is that: first we establish an upper bound for a fixed $\w\in\N(\W,\varepsilon)$ using Proposition~\ref{lem:con} and then use the union bound and the covering number of $\W$ to get an upper bound for any $\w\in\N(\W,\varepsilon)$. Then we utilize the property of the $\varepsilon$-net to prove the inequality in the lemma. %Note that if the added regularizer can restrict $\wh$ to a much smaller set than $\W$ (e.g., a sparse set), we could leverage the the covering number of the smaller set to derive better bound in the above lemma and the main theorems. 

\subsection{Proof of Theorem~\ref{thm:smooth:convex}}
Let $\varepsilon>0$ and $H = I + \frac{\sigma}{\alpha}\E[\nabla f(\w_*, \z)\nabla f(\w_*, \z)^{\top}]$ with $\alpha>0$. The values of $\alpha$ and $\varepsilon$ will be decided later. 
\begin{align*}
 P(\wh) - P(\w_*)  &\leq   \nabla P(\wh)^{\top}( \wh - \w_*)=  (\nabla P(\wh) - \nabla P(\w_*))^{\top}(\wh - \w_*) +  \nabla P(\w_*)^{\top}( \wh - \w_*) \\
&= (\nabla P(\wh) - \nabla P(\w_*) - [\nabla P_n(\wh) - \nabla P_n(\w_*)])^{\top} (\wh - \w_* ) \\
&+( \nabla P(\w_*) - \nabla P_n(\w_*))^{\top}( \wh - \w_* ) + \nabla P_n(\wh)^{\top}(\wh - \w_* ) \\
&\leq (\nabla P(\wh) - \nabla P(\w_*) - [\nabla P_n(\wh) - \nabla P_n(\w_*)])^{\top}( \wh - \w_* )\\
& +(\nabla P(\w_*) - \nabla P_n(\w_*)^{\top} (\wh - \w_*),
\end{align*}
where the first inequality uses the convexity of $P(\w)$ and  the second inequality uses the optimality condition of $\wh$, i.e.,  $ \langle \nabla P_n(\wh), \wh - \w_*\rangle\leq 0$. 
Then we have
\begin{align*}
P(\wh) - P(\w_*)&\leq \left\| \nabla P(\wh) - \nabla P(\w_*) - [\nabla P_n(\wh) - \nabla P_n(\w_*)]\right\|_2 \|\wh - \w_*\|_2\\
&+ \left\|\nabla P(\w_*) - \nabla P_n(\w_*)\right\|_{H^{-1}}\|\wh - \w_*\|_{H},
\end{align*}

By Lemma~\ref{lem:net} and Lemma~\ref{lem:vec:con}, with probability at least $1-2\delta$, we have
\begin{equation} \label{eqn:add:3}
\begin{split}
& P(\wh) - P(\w_*)  \leq  2L \varepsilon \left\| \wh - \w_* \right\|_2   + \frac{LC(\varepsilon)\|\wh - \w_*\|_2^2 }{n}   +  \frac{LC(\varepsilon) \varepsilon \|\wh - \w_*\|_2}{n}\\
&  +  \left\| \wh - \w_* \right\|_2 \sqrt{\frac{LC(\varepsilon)(P(\wh)-  P(\w_*) )}{n}}  +\left\| \wh - \w_* \right\|_2 \sqrt{\frac{LGC(\varepsilon) \varepsilon}{n}}\\
& +  \frac{2G\log(2/\delta)\left\| \wh - \w_* \right\|_{H} }{n}  +   \left\| \wh - \w_* \right\|_{H} \sqrt{\frac{2\alpha d \log(2/\delta)}{n\sigma}}.
\end{split}
\end{equation}
Next, we bound the last four terms in the R.H.S using H\"{o}lder's inequality.  \begin{align}
&\left\| \wh - \w_* \right\|_2 \sqrt{\frac{LC(\varepsilon)(P(\wh)-  P(\w_*))}{n}}  \leq \frac{3LC(\varepsilon) \left\| \wh - \w_* \right\|_2^2}{ 2n} + \frac{P(\wh)-  P(\w_*) }{6}, \label{eqn:inequality:1} \\
&\left\| \wh - \w_* \right\|_2 \sqrt{\frac{LC(\varepsilon) G \varepsilon}{n}} \leq \frac{LC(\varepsilon) \left\| \wh - \w_* \right\|_2^2}{ n} +  \frac{G \varepsilon}{4},\label{eqn:inequality:2}\\
& \left\| \wh - \w_* \right\|_{H} \sqrt{\frac{2\alpha d \log(2/\delta)}{n\sigma}}\leq \frac{\alpha}{12}\|\w_* - \wh\|_{H}^2 + \frac{6d\log(2/\delta)}{n\sigma},\label{eqn:inequality:3}\\
&  \frac{2G\log(2/\delta)\left\| \wh - \w_* \right\|_{H} }{n}\leq  \frac{\alpha}{12}\|\w_*- \wh\|_H^2 + \frac{12G^2\log^2(2/\delta)}{n^2\alpha}\label{eqn:inequality:4}.
\end{align}

Combining the inequalities in~(\ref{eqn:add:3}),~(\ref{eqn:inequality:1}),~(\ref{eqn:inequality:2}), (\ref{eqn:inequality:3}),  and~(\ref{eqn:inequality:4}), with probability $1-2\delta$ we have
\begin{align*}
P(\wh) - P(\w_*)&\leq  2L \varepsilon \left\| \wh - \w_* \right\|_2+ \frac{4LC(\varepsilon)\|\wh - \w_*\|_2^2 }{n} +  \frac{LC(\varepsilon) \varepsilon \|\wh - \w_*\|_2}{n}  \\
&  +  \frac{G \varepsilon}{4} + \frac{P(\wh) - P(\w_*)}{6} + \frac{\alpha}{6}\|\wh - \w_*\|_H^2  +  \frac{12G^2\log^2(2/\delta)}{n^2\alpha} +  \frac{6d\log(2/\delta)}{n\sigma}     \\
&\leq 2L \varepsilon \left\| \wh - \w_* \right\|_2 + \frac{4LC(\varepsilon)\|\wh - \w_*\|_2^2 }{n} +  \frac{LC(\varepsilon) \varepsilon \|\wh - \w_*\|_2}{n} \\
&+  \frac{G \varepsilon}{4}   + \frac{P(\wh) - P(\w_*)}{2} + \frac{\alpha}{6}\|\wh - \w_*\|_2^2+  \frac{12G^2\log^2(2/\delta)}{n^2\alpha} +  \frac{6d\log(2/\delta)}{n\sigma},  
\end{align*}
where the second inequality uses Lemma~\ref{lem:0}. 
Then we have
\begin{align*}
\frac{P(\wh) - P(\w_*)}{2}&\leq  \frac{4LC(\varepsilon)\|\wh - \w_*\|_2^2 }{n} +  \frac{LC(\varepsilon) \varepsilon \|\wh - \w_*\|_2}{n} \\
&+  \frac{6 d\log(2/\delta)}{n\sigma}    +  \frac{12G^2\log^2(2/\delta)}{n^2\alpha} + \frac{G\varepsilon}{4} + 2L\varepsilon\|\wh - \w_*\|_2 + \frac{\alpha}{6}\|\wh  - \w_*\|_2^2\\
\end{align*}
Let $\varepsilon= \frac{1}{n}$, $\alpha = \frac{\log(2/\delta)}{n}$ and noting that $C(\varepsilon) = 2(\log(2/\delta) + d \log(6Rn))$, $\|\w - \w_*\|_2\leq 2R$, we have
\begin{align*}
&P(\wh) - P(\w_*)\leq  \frac{64LR^2(\log(2/\delta) + d\log(6Rn)) }{n} +  \frac{8LR(\log(2/\delta) + d\log(6Rn))}{n^2} \\
&+  \frac{12 d\log(2/\delta)}{n\sigma}    +  \frac{24G^2\log(2/\delta)}{n} + \frac{G}{2n} +\frac{ 8LR }{n}+ \frac{4R^2\log(2/\delta)}{3n}=O\left(\frac{d\log n}{n} + \frac{d\log(1/\delta)}{n\sigma}\right).
\end{align*}

\subsection{Proof of Theorem~\ref{thm:smooth:convex2}}
We utilize the Theorem~\ref{thm:smooth:convex} to prove Theorem~\ref{thm:smooth:convex2}. First, we define and recall some notations. 
\begin{align*}
&\Fh(\w) = F(\w) + \frac{1}{n}g(\w), \quad \Fh_n(\w) = F_n(\w) + \frac{1}{n}g(\w)\\
&\wh_* = \arg\min_{\w\in\W}\Fh(\w), \quad \wt = \arg\min_{\w\in\W}\Fh_n(\w), \quad \w^* = \arg\min_{\w\in\W}F(\w)
\end{align*}
where $F(\w)=\E[f(\w,\z)]$ and $F_n(\w) = \frac{1}{n}\sum_{i=1}^n f(\w, \z_i)$. According to Theorem~\ref{thm:smooth:convex}, the following inequality holds with high probability $1-\delta$, 
\begin{align*}
\Fh(\wt) - \Fh(\wh_*)\leq O\left(\frac{d\log n}{n} + \frac{d\log(1/\delta)}{n\sigma}\right).
\end{align*}
Plugging the definition of $\Fh$ we have
\begin{align}\label{eqn:com1}
F(\wt) - F(\wh_*)\leq \frac{1}{n}(g(\wh_*) - g(\wt))+ O\left(\frac{d\log n}{n} + \frac{d\log(1/\delta)}{n\sigma}\right)\leq O\left(\frac{d\log n}{n} + \frac{d\log(1/\delta)}{n\sigma}\right)
\end{align}
where the last inequality uses the assumption $\sup_{\w, \w'}|g(\w) - g(\w')|\leq B$. 

Due to that $\wh_*$ is the minimizer of $\Fh(\w)$, then 
%By the optimality condition of $\wh_*$ and the  strong convexity of $g(\w)$ whose strong convexity modulus is equal to $1$, we have for any $\w\in\W$
%\begin{align*}
%\Fh(\w)\geq \Fh(\wh_*) + \frac{1}{2n}\|\w - \wh_*\|_2^2
%\end{align*}
%Let $\w=\w^*$ in the above inequality, we have
\begin{align*}
 F(\wh_*) + \frac{1}{n}g(\wh_*) \leq F(\w^*) + \frac{1}{n}g(\w^*)
\end{align*}
Then 
\begin{align}\label{eqn:com2}
F(\wh_*)\leq F(\w^*) + \frac{1}{n}(g(\w^*) - g(\wh_*))\leq  F(\w^*)  + \frac{B}{n}
\end{align}
Combining~(\ref{eqn:com1}) and~(\ref{eqn:com2}), the following inequality holds with high probability $1-\delta$
\begin{align*}
F(\wt) - F(\w^*)\leq O\left(\frac{d\log n}{n} + \frac{d\log(1/\delta)}{n\sigma}\right)
\end{align*}
which finishes the proof. 

\section{Conclusion}
In this paper, we have developed a simple analysis of fast rats for  empirical minimization with exponential concave loss functions and  an arbitrary convex regularizer. This represents the first result of its kind. The proof is elementary only exploiting  the covering number of a finite-dimensional bounded set and a concentration inequality of random vectors. Our framework also induces a unified fast rate results for exponential concave empirical risk minimization without and with  any convex regularizer. An open problem remains  is  whether the $\log(n)$ factor can be removed without using the boosting technique.

\bibliographystyle{abbrv}
\bibliography{ref,all}

\appendix
\section{Proof of Lemma~\ref{lem:net}}\label{sec:lem}
The proof  is similar to the proof of Lemma 1 in~\cite{DBLP:journals/corr/0005YJ17}.
Denote by $\N(\W, \varepsilon)$ the $\varepsilon$-net of $\W$ with minimal cardinality. By the covering number theory, we have $\log |\N(\W, \varepsilon)|= N(\W, \varepsilon)\leq  d\log(6R/\varepsilon)$.  To prove the upper bound for all $\w\in\W$, we first consider a fixed point  in the $\N(\W, \varepsilon)$ denoted by $\w\in\N(\W, \varepsilon)$. Since $f(\cdot, \z)$ is $L$-smooth for any $\z\in\Z$, we have
\begin{equation} \label{eqn:smooth:lemma}
\left\| \nabla f(\w, \z) - \nabla f(\w_*, \z) \right\|_2\leq 
 L \|\w - \w_*\|_2.
\end{equation}
Because $f(\cdot, \z)$ is both convex and $L$-smooth, by (2.1.7) of \cite{nesterov2004introductory}, we have
\[
\left\| \nabla f(\w, \z) - \nabla f(\w_*, \z) \right\|_2^2  \leq 2L \left(f(\w, \z)-  f(\w_*, \z)  - \langle \nabla f(\w_*, \z), \w-\w_* \rangle \right).
\]
Taking expectation over both sides, we have
\begin{align*}
 \E \left[ \left\| \nabla f(\w, \z) - \nabla f(\w_*, \z) \right\|_2^2\right]
&\leq  2L \left(F(\w)-  F(\w_*)  - \langle \nabla F(\w_*), \w-\w_* \rangle \right)\\
&\leq  2L \left(F(\w)-  F(\w_*)  - (R(\w_*) - R(\w)) \right) = 2L(P(\w) - P(\w_*)),
\end{align*}
where the second inequality uses Lemma~\ref{lem:opt}. 
Following Proposition~\ref{lem:con},  with probability at least $1-\delta$, we have
\[
\begin{split}
& \left\| \nabla P(\w) - \nabla P(\w_*) - [\nabla P_n(\w) - \nabla P_n(\w_*)] \right\|_2 \\
&= \left\| \nabla F(\w) - \nabla F(\w_*) - \frac{1}{n} \sum_{i=1}^n [ \nabla f(\w, \z_i) - \nabla f(\w_*, \z_i)] \right\| \\
 &\leq  \frac{2 L \|\w - \w_*\|_2 \log(2/\delta)}{n} + \sqrt{\frac{4 L(P(\w)-  P(\w_*) ) \log(2/\delta)}{n}}.
\end{split}
\]
By taking the union bound over all $\w \in \N(\W , \varepsilon)$,  
%with a probability  at least $1-|\N(\W, \varepsilon)|\delta$, the following holds for all $\w\in\N(\W,\varepsilon)$
%\begin{align*}
%&\left\| \nabla P(\w) - \nabla P(\w_*) - [\nabla P_n(\w) - \nabla P_n(\w_*)] \right\|_2 \\
%&\leq \frac{2 L \|\w - \w_*\|_2 \log(2/\delta)}{n} + \sqrt{\frac{2 L(P(\w)-  P(\w_*) ) \log(2/\delta)}{n}}.
%\end{align*}
%In other word, 
with probability at least  $1-\delta$, the following holds for all $\w\in\N(\W, \varepsilon)$, 
\begin{align}\label{eqn:anyw}
&\left\| \nabla P(\w) - \nabla P(\w_*) - [\nabla P_n(\w) - \nabla P_n(\w_*)] \right\|_2 \\
&\leq \frac{2 L \|\w - \w_*\|_2\left( \log(2/\delta) + \log|\N(\W,\varepsilon)|\right)}{n} + \sqrt{\frac{4 L(P(\w)-  P(\w_*) )\left( \log(2/\delta) + \log|\N(\W,\varepsilon)|\right)}{n}}\notag\\
&\leq \frac{L \|\w - \w_*\|_2C(\varepsilon)}{n} + \sqrt{\frac{ L(P(\w)-  P(\w_*) )C(\varepsilon)}{n}}.\notag
\end{align}
Next, we consider any $\w\in\W$. There exists $\w'\in\N(\W, \varepsilon)$ such that $\|\w - \w'\|_2\leq\varepsilon$. Then 
\begin{align*}
&\left\| \nabla P(\w) - \nabla P(\w_*) - [\nabla P_n(\w) - \nabla P_n(\w_*)] \right\|_2 \\
&\leq \left\| \nabla P(\w') - \nabla P(\w_*) - [\nabla P_n(\w') - \nabla P_n(\w_*)] \right\|_2\\
& + \left\| \nabla P(\w) - \nabla P(\w_*)\right\|_2 +\left\| [\nabla P_n(\w) - \nabla P_n(\w')] \right\|_2  \\
&\leq \underbrace{ \left\| \nabla P(\w') - \nabla P(\w_*) - [\nabla P_n(\w') - \nabla P_n(\w_*)] \right\|_2}\limits_{A}+ 2L\varepsilon,
\end{align*}
where the last inequality follows that both $P(\w)$ and $P_n(\w)$ are $L$-smooth functions. We can bound the term $A$ using~(\ref{eqn:anyw}), then we have with probability at least $1-\delta$, 
\begin{align*}
&\left\| \nabla P(\w) - \nabla P(\w_*) - [\nabla P_n(\w) - \nabla P_n(\w_*)] \right\|_2 \\
&\leq\frac{L \|\w' - \w_*\|_2C(\varepsilon)}{n} + \sqrt{\frac{ L(P(\w')-  P(\w_*) )C(\varepsilon)}{n}}+ 2L\varepsilon\\
&\leq\frac{L( \|\w - \w_*\|_2 + \|\w - \w'\|_2)C(\varepsilon)}{n} + \sqrt{\frac{ L(P(\w)-  P(\w_*)  + |P(\w) - P(\w')|)C(\varepsilon)}{n}}+ 2L\varepsilon\\
&\leq\frac{L( \|\w - \w_*\|_2 + \varepsilon)C(\varepsilon)}{n} + \sqrt{\frac{ L(P(\w)-  P(\w_*)  + G\varepsilon)C(\varepsilon)}{n}}+ 2L\varepsilon.
\end{align*}
By rearranging the terms, we can finish the proof. 

\end{document}